\documentclass[11pt]{article}

\usepackage[preprint]{acl}
\usepackage{times}
\usepackage{latexsym}
\usepackage[T1]{fontenc}
\usepackage[utf8]{inputenc}
\usepackage{microtype}
\usepackage{inconsolata}
\usepackage{graphicx}
\usepackage{amsmath,amssymb}
\usepackage{amsthm}
\usepackage{booktabs}
\usepackage{multirow}
\usepackage{tabularx}
\usepackage{url}
\usepackage{algorithm}
\usepackage{algpseudocode}
\newcolumntype{Y}{>{\raggedright\arraybackslash}X}

\newtheorem{theorem}{Theorem}
\newtheorem{proposition}{Proposition}

\title{Beyond Local Accuracy: A Protocol-Level Identifiability Audit for Controlled LLM Reasoning Evaluation}

\author{
  Junhao Luo \quad
  Ning Huang\thanks{These authors contributed equally as second authors.} \quad
  Ziqi Sha\textsuperscript{*} \quad
  Wenxuan Tang\textsuperscript{*} \quad
  Wei Deng\thanks{Corresponding author: \texttt{dengwei@swufe.edu.cn}.} \\
  School of Statistics and Data Science, Southwestern University of Finance and Economics \\
  Chengdu, China
}

\begin{document}
\begin{NoHyper}
\maketitle
\end{NoHyper}

\begin{abstract}
LLM benchmark scores can be precise even when the observation protocol does not identify the behavioral property they are intended to measure. In a controlled, solver-grounded setting, we formalize a protocol-level identifiability audit over a finite behavioral policy class: given policies $\mathcal{H}$, observation support $O$, and estimand $\tau$, we test whether $O$ separates every pair with different $\tau$. The audit requires zero model calls and resolves our diagnostic case: base-only observation collapses seven frozen deterministic policies into one equivalence class; full support yields seven classes and no cross-estimand collisions; every leave-one-out support retains a constructive collision witness. Empirically, both constrained-generation variants have pair-validity 1.0, yet base accuracy and selective-response fidelity diverge---0.620 versus 0.324 across six balanced oracle-transition directions (cluster-bootstrap 95\% CI [0.600, 0.642] vs.\ [0.304, 0.345])---and the gap recurs on a second deterministic source (0.646 vs.\ 0.331). The audit also synthesizes a minimum identifying support $O^*$ for the frozen policy class: two cells instead of the full 36-cell tensor. This case shows how evaluation-design validity can be checked structurally before model inference and why base correctness does not determine intervention-response fidelity.
\end{abstract}

\section{Introduction}

LLM evaluation increasingly relies on accuracy over static benchmarks as a proxy for reasoning quality, despite longstanding concerns that broad capability claims can outrun what benchmark measurements support \citep{raji-etal-2021-everything,salaudeen-etal-2025-measurement}. A model that answers base queries correctly is often described as ``faithful,'' ``sound,'' or ``capable,'' and its failure modes are read off from where accuracy drops. This practice conflates two different quantities. \emph{Static correctness} asks whether the model's answer matches an oracle on fixed inputs. \emph{Intervention-response fidelity} asks whether the model's answer changes in the direction prescribed by a controlled intervention on the input. These are different estimands, and nothing guarantees that a protocol that measures one also measures the other.

This conflation is not merely theoretical. The fragility of reported accuracy under controlled re-instantiation is well documented for standard LLM reasoning benchmarks: chain-of-thought prompting unlocks high accuracy on GSM8K \citep{wei-etal-2022-cot}, yet symbolic re-instantiation with surface-level edits drops accuracy by up to 65\% even though the added clauses do not change the answer \citep{mirzadeh-etal-2024-gsm-symbolic}. Counterfactual task variants likewise reveal systematic degradation when default task assumptions are altered, distinguishing transferable abstract reasoning from narrower task-specific procedures \citep{wu-etal-2024-reasoning}. That work diagnoses counterfactual fragility empirically; we instead ask whether a protocol's observation support point-identifies its claimed estimand on a declared policy class. Recent audits of hundreds of LLM benchmarks find that many operationalize their target constructs through tasks and metrics whose validity is left unjustified \citep{bean2025measuring}, and that the distributional assumptions underlying standard aggregation can change model rankings \citep{siska-etal-2024-examining}. Work on chain-of-thought faithfulness has shown that generated explanations need not reflect the reasoning actually used \citep{turpin-etal-2023-language,lanham-etal-2023-measuring}, while counterfactual prompting studies emphasize that an intervention's effect cannot be attributed to the targeted factor without a semantic-preserving baseline \citep{yang-etal-2026-compared}. What these lines share is a growing recognition that measurement validity depends on what an evaluation protocol is capable of distinguishing. Yet the question is rarely asked in identification terms: \emph{given the observations a protocol collects, is the target estimand point-identified at all?}

We study this distinction in a controlled setting: solver-grounded three-valued reasoning over open-world signed-Horn theories. An independent solver fixes, before any model run, the oracle label for each intervention world. Models answer under a representation contract (TRUE/FALSE/UNKNOWN labels), a readout channel (generated choice or candidate scoring), and a label mapping (one of the six $S_3$ permutations of TRUE/FALSE/UNKNOWN to A/B/C). We ask: \emph{when does the observation support of an evaluation protocol suffice to identify a target behavioral estimand?} This setting serves as a controlled testbed for protocol diagnosis rather than evidence of unrestricted transport to natural-language evaluation.

Our contribution is a protocol-level identifiability audit: before interpreting a reported score, test whether the protocol collects observations that separate the target estimand from behaviorally distinct alternatives. Identification is defined relative to the declared estimand, policy class, and observation support.

\begin{itemize}
\item \textbf{Finite-class protocol identifiability criterion.} We define behavioral policy classes, observation equivalence, and point identification, and prove that a target estimand is point-identified exactly when policies with different estimand values cannot collide on the observed support---relative to the declared policy class and support (Section~\ref{sec:framework}).
\item \textbf{Collision-based synthesis of minimum identifying support.} The same collision structure that detects under-identification synthesizes a minimal identifying support $O^*$: an exact minimum hitting set over cross-estimand distinguishing cells. In the frozen seven-policy case this synthesis yields 26 minimum supports of size two (none mandatory), versus 36 for the full tensor; the numerical minimum is case-specific, while the synthesis procedure is the methodological contribution (Section~\ref{sec:synthetic}).
\item \textbf{Controlled empirical demonstration.} On frozen responses from two instruction-tuned models, we show that local accuracy and intervention-response fidelity diverge under pair-valid constrained generation, across six balanced oracle-transition directions and a deterministic second source (Sections~\ref{sec:design}--\ref{sec:real}).
\end{itemize}

The central claim is structural and protocol-relative: identification holds on the frozen policy class and support, and its measurement consequences hold across readout contracts, transition directions, and two deterministic sources.

\begin{figure*}[t]
\centering
\includegraphics[width=0.96\textwidth]{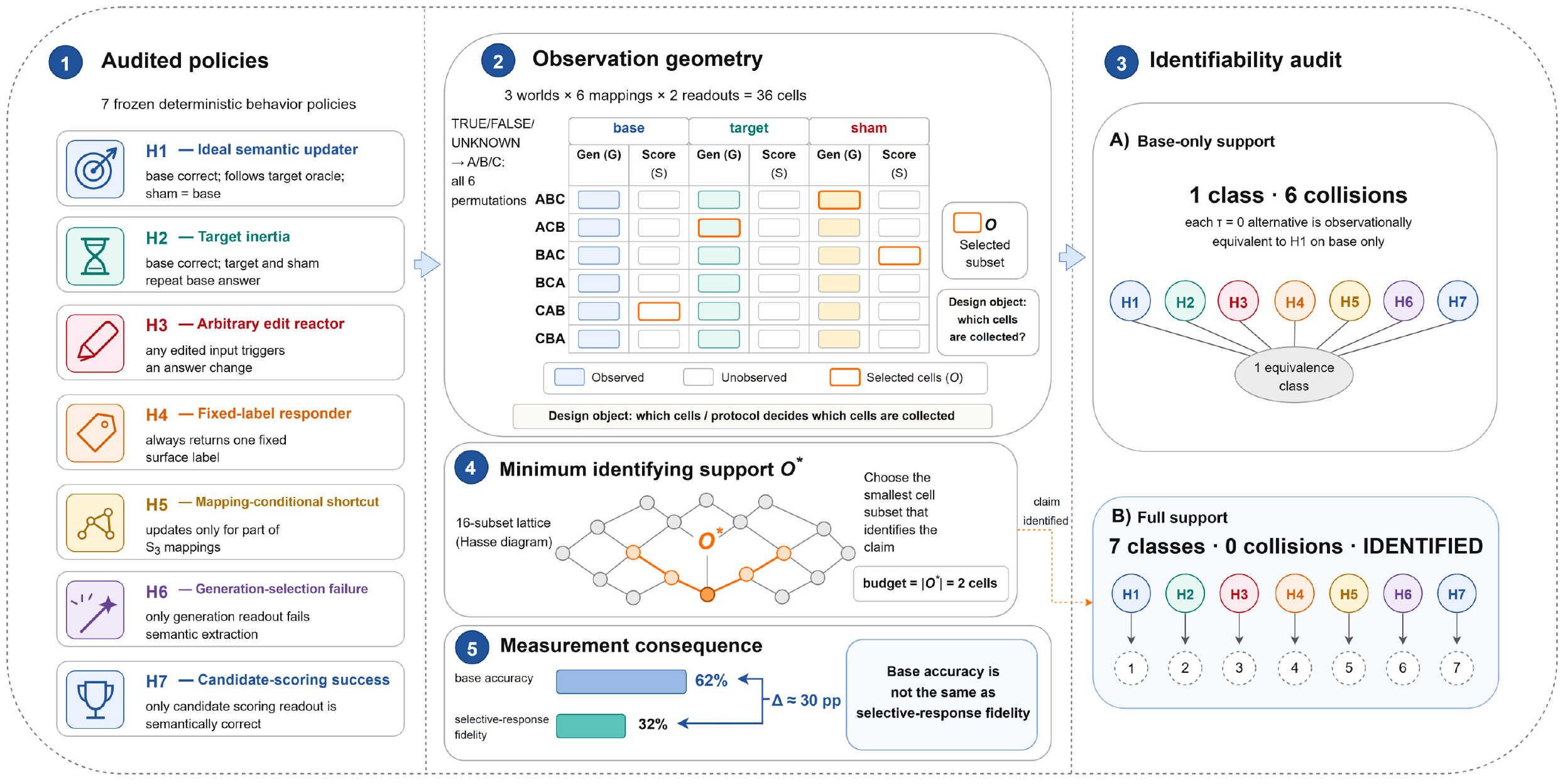}
\caption{Protocol design overview. (1) \emph{Audited policies:} seven frozen deterministic behaviors H1--H7 (ideal semantic updater, target inertia, arbitrary edit reactor, fixed-label responder, mapping-conditional shortcut, generation-selection failure, candidate-scoring success) with their base/target/sham response rules. (2) \emph{Observation geometry:} the world$\times$mapping$\times$readout tensor with $3\times6\times2=36$ cells; orange marks the selected support $O$. (3) \emph{Identifiability audit:} base-only support yields 1 class with 6 cross-estimand collisions (every $\tau{=}0$ alternative is observationally equivalent to the ideal updater on base alone); full support yields 7 classes with 0 collisions (IDENTIFIED). (4) \emph{Minimum identifying support $O^*$:} the smallest cell subset that identifies the claim---just 2 cells on this policy class (budget $=|O^*|$). (5) \emph{Empirical measurement consequence:} base accuracy (62\%) and selective-response fidelity (32\%) are distinct quantities ($\Delta\approx30$\,pp).}
\label{fig:overview}
\end{figure*}

\section{Formal Setup: Protocol-Level Identification}
\label{sec:framework}

We cast evaluation design as an \emph{identification} problem: before asking how accurately a protocol estimates a target behavioral quantity, we ask whether the protocol's observation support is in principle sufficient to identify that quantity. This separates \emph{design validity} (can the support distinguish the estimand's values?) from \emph{estimation accuracy} (how noisy is the estimate on a finite sample?), and it lets us diagnose under-identification structurally---without running any model.

\subsection{Behavioral policies, observation support, and two estimands}
Let $\mathcal{H}$ be a finite class of deterministic behavioral policies. A policy $h\in\mathcal{H}$ maps an input under a representation contract to a response. An \emph{observation support} $O$ is the set of responses a protocol collects. The theoretical target is the \emph{full-support contract-stable selective-response property} $\tau_{\mathrm{full}}:\mathcal{H}\to\{0,1\}$: it equals one only when a policy is base-correct, follows the target oracle, remains stable on matched sham, and does so across both readouts and all six mappings. This binary property is the estimand used by the theorem and synthetic recovery.

The empirical pilot also reports a distinct \emph{local paired selective-response rate}, $\theta_{\mathrm{local}}$, over units $(\text{model},\text{cluster},\pi,r)$. A unit succeeds when its base and target responses are valid and oracle-correct and its sham response is valid and semantically equals its paired base response. Thus $\theta_{\mathrm{local}}$ averages contract-local successes, whereas the full-support empirical criterion requires every mapping--readout contract to succeed within a model--cluster. Neither denominator nor estimand level is interchangeable.

Two policies $h,h'$ are \emph{observationally equivalent} on $O$, written $h\equiv_O h'$, if they produce identical observable responses on $O$. Any estimator based only on $O$ must assign the same value to equivalent policies.

\subsection{Finite-class identification criterion}
\begin{theorem}[Point identification on finite behavior classes]
\label{thm:id}
$\tau$ is point-identified on $O$ if and only if for all $h,h'\in\mathcal{H}$, $\tau(h)\neq\tau(h')$ implies $h\not\equiv_O h'$.
\end{theorem}

\begin{proof}
($\Rightarrow$) If some $h,h'$ with $\tau(h)\neq\tau(h')$ satisfy $h\equiv_O h'$, then any estimator $\hat\tau$ measurable with respect to $O$ assigns $\hat\tau(h)=\hat\tau(h')$, so at least one is wrong. ($\Leftarrow$) If all pairs with different $\tau$ are inequivalent, then $\tau$ is constant on each equivalence class, so the estimator $\hat\tau(h)=\tau(h^\ast)$ for any $h^\ast\equiv_O h$ is well-defined and correct.
\end{proof}

The theorem is deliberately finite and non-statistical: it states a structural condition on the support and the policy class, independent of sample size. Statistical uncertainty (Section~\ref{sec:real}) sits on top of identification, not in place of it---a protocol that is not point-identified cannot be rescued by more data, only by a larger support.

\subsection{Interventional response tensor}
We organize observations as an \emph{Interventional Response Tensor} $R_{w,\pi,r}$. Its indices are world $w$, label permutation $\pi$, and readout $r$. Worlds are base, target, and sham; $\pi\in S_3$; readouts are generated choice and candidate scoring. Each cell stores the raw response, validity, decoded semantic response, and solver-grounded oracle. The tensor has three indices, while the protocol requires target and matched-sham observations, paired readouts, and all six mappings. These are four diagnostically distinct support components, not four orthogonal tensor axes. Full support has $3\times6\times2=36$ cells per model--cluster.

\subsection{Corollaries: axis insufficiency}
Theorem~\ref{thm:id} yields four corollaries, each obtained by exhibiting a pair of policies with different $\tau$ that are equivalent when a support component is removed. We state them informally; Section~\ref{sec:synthetic} constructs the witnesses.
\begin{description}
\item[Base insufficiency.] $O=\{\text{base}\}$ cannot separate a selective updater from target inertia (both answer base correctly).
\item[Sham insufficiency.] Dropping matched sham equates a selective updater with an any-edit reactor (both update on target; only sham distinguishes them).
\item[Readout insufficiency.] Dropping paired readout equates a semantic failure with a contract-specific reportability failure (a model that fails on one readout is indistinguishable from one that fails semantically on both).
\item[Mapping insufficiency.] Dropping full $S_3$ equates a semantic updater with an unseen-mapping label shortcut (a model correct on the observed mapping is indistinguishable from one that hard-codes labels).
\end{description}
Each corollary is a constructive collision, not a heuristic: the synthetic witnesses of Section~\ref{sec:synthetic} realize each pair explicitly.

\section{Protocol Support and Component Essentiality}
\label{sec:design}

\subsection{Solver-grounded three-valued semantics}
The pilot uses open-world signed-Horn theories under three-valued (TRUE/FALSE/UNKNOWN) semantics. An independent solver fixes, before any model run, the oracle label for each world by checking whether the query is entailed, refuted, or undetermined by the theory. This solver grounding is essential: it removes the experimenter's judgment from the oracle, so the target oracle is a property of the theory plus the intervention, not of a model annotation.

\subsection{World construction}
Three worlds are constructed per cluster. The \emph{base} world is the original input. The \emph{target} world applies an edit that changes the solver oracle, while the \emph{sham} world applies an edit that preserves it. Within the frozen policy class, observing both roles separates oracle-directed response from an any-edit reactor. This is a protocol-discrimination role, not clean causal isolation: target uses support replacement whereas sham adds a query-disjoint fact, so edit type perfectly distinguishes the roles and length, edit-distance, and lexical cues also differ. Formal oracle validation has zero mismatches over 144 worlds, and the automatic surface audit flags material surface confounding. For the human semantic audit, two reviewers independently annotated all target and sham items from 48 balanced clusters (96 items total). Each reviewer received an independently randomized order with role, expected label, and model output hidden, and assigned TRUE, FALSE, UNKNOWN, or UNSURE. They agreed on 96/96 items (Cohen's $\kappa=1.0$), so no adjudication was required. This audit checks semantic labels, not causal comparability. Consequently, target/sham contrasts identify behavior under these protocol conditions, not a pure semantic-edit effect.

\subsection{Four support components}
The frozen audit varies four diagnostically distinct support components:
\begin{itemize}
\item \textbf{target} separates selective updating from retaining a base response after the oracle changes;
\item \textbf{matched sham} separates oracle-directed updating from responding to any input edit;
\item \textbf{paired readout} reveals whether validity and diagnosis depend on the response channel;
\item \textbf{full $S_3$ mapping support} reveals dependence on a subset of surface-label contracts.
\end{itemize}
These are protocol requirements over three tensor indices: target and sham are distinct world roles within $w$, while readout and mapping correspond to $r$ and $\pi$. The labels describe behavioral distinctions in observed responses and do not assert internal causes.

\subsection{Axis-essentiality proposition}
\begin{proposition}
\label{prop:axis}
If policies $h,h'$ exist with $\tau(h)\neq\tau(h')$ that are separable under the full support but collide when a support component $a$ is removed, then component $a$ is essential for the policy class and estimand.
\end{proposition}

The complete subset lattice (Section~\ref{sec:synthetic}) supplies constructive evidence for all four components: each leave-one-out support retains a collision between the ideal updater ($\tau_{\mathrm{full}}=1$) and at least one $\tau_{\mathrm{full}}=0$ alternative. Removing any component therefore destroys point identification on this policy class.

\section{Synthetic Policy Recovery}
\label{sec:synthetic}

This section performs \emph{zero model calls}: it validates that the measurement design and the analyzer can recover behavioral differences that are known by construction, separating design validity from any property of a particular LLM. Concretely, we ask whether the observation support $O$ suffices to point-identify the target estimand $\tau$ (selective-response fidelity) on a finite, frozen policy class $\mathcal{H}$, and we answer by exhaustive enumeration rather than by estimation.

\subsection{Frozen policy class}
We implement seven deterministic policies plus one stochastic mixture. The reference \emph{ideal semantic updater} has $\tau_{\mathrm{full}}{=}1$; the other six deterministic policies have $\tau_{\mathrm{full}}{=}0$. They cover inertia and indiscriminate reaction (\emph{target inertia}, \emph{any-edit reactor}); label-based shortcuts (\emph{fixed-label responder}, \emph{mapping-conditional shortcut}); and readout-specific failures (\emph{generated-only failure}, \emph{scoring-only success}). The separate \emph{stochastic mixture} mixes ideal and inertia behavior at a declared rate and is analyzed only at its stated granularity.

\subsection{Complete subset lattice and one addition path}
We enumerate the complete $2^4=16$ subset lattice over target, sham, paired readout, and full $S_3$ mapping support. For readability, Table~\ref{tab:synthetic} and Figure~\ref{fig:synthetic} show one fixed cumulative addition path, $O_0$ through $O_4$; the path is illustrative rather than a unique ordering. The leave-one-out results in Table~\ref{tab:loo} establish component necessity without relying on that ordering, and Appendix~\ref{app:synthetic-full} reports all 16 subsets.

\begin{table}[t]
\centering
\footnotesize
\setlength{\tabcolsep}{4pt}
\begin{tabular}{llrrll}
\toprule
Support & Axis added & $|O|$ & Classes & Coll. & Identified \\
\midrule
$O_0$ & base          &  1 & 1 & 6 & no \\
$O_1$ & + target      &  2 & 2 & 3 & no \\
$O_2$ & + sham        &  3 & 3 & 2 & no \\
$O_3$ & + readout     &  6 & 5 & 1 & no \\
$O_4$ & + full $S_3$  & 36 & 7 & 0 & \textbf{yes} \\
\bottomrule
\end{tabular}
\caption{One fixed addition path through the 16-subset lattice. $|O|$ is projection size, ``Classes'' counts observation-equivalence classes among seven deterministic policies, and ``Coll.'' counts cross-estimand collision pairs. Only full support $O_4$ point-identifies $\tau_{\mathrm{full}}$ on the frozen policy class.}
\label{tab:synthetic}
\end{table}

\begin{figure*}[t]
\centering
\includegraphics[width=\textwidth]{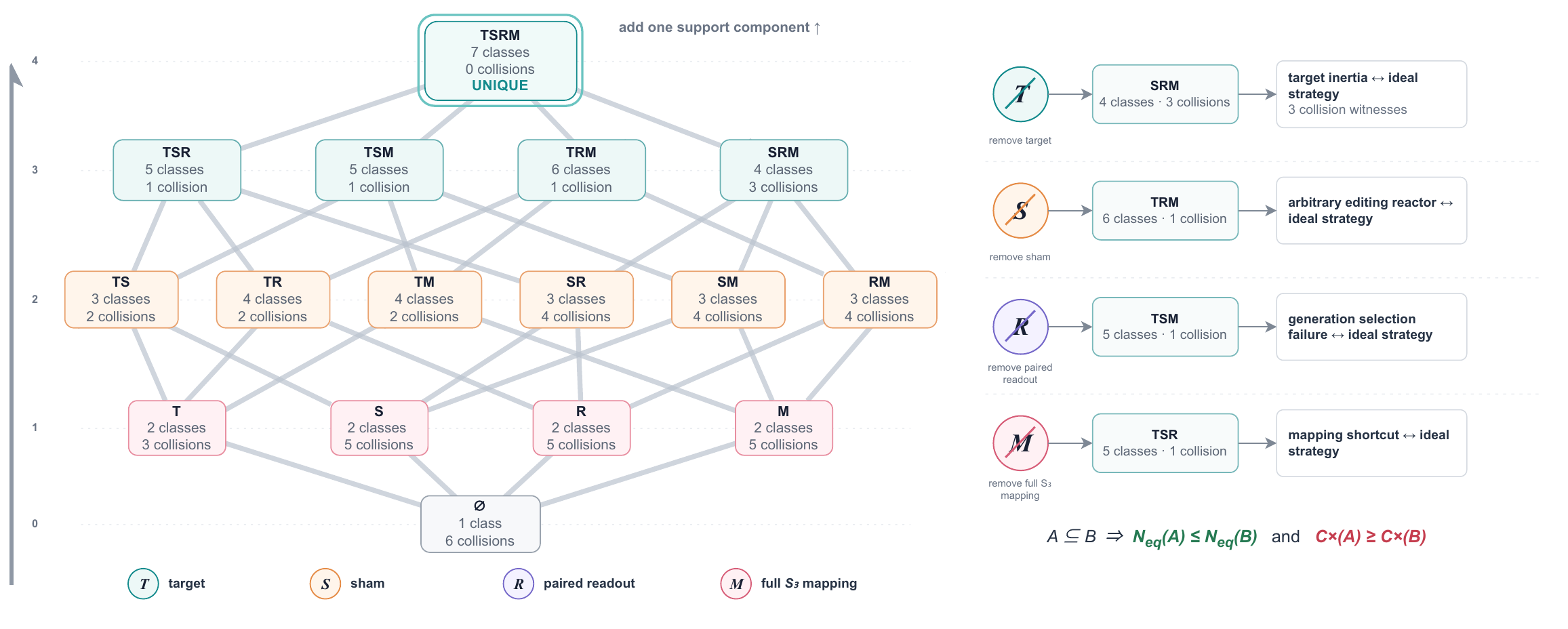}
\caption{Support-component essentiality: the full power-set lattice over target ($T$), sham ($S$), paired readout ($R$), and full-$S_3$ mapping ($M$). Each node reports its observation-equivalence class count and cross-estimand collision count over the seven frozen policies. Only full support $TSRM$ is identified (7 classes, 0 collisions; \textsc{unique}). The four rank-3 nodes are the leave-one-out witnesses: removing any component reintroduces at least one collision (e.g., removing target leaves $SRM$ with 3 collisions between target inertia and the ideal updater); removing more components only adds collisions and merges classes.}
\label{fig:synthetic}
\end{figure*}

\subsection{Recovery results}
Under $O_0$ (base only), all seven deterministic policies project to a \emph{single} observation-equivalence class, and the estimand is not point-identified: six collision pairs link the ideal updater to every $\tau{=}0$ alternative. Each added axis splits classes: $O_1$ yields 2 classes (collapses 3 collisions), $O_2$ yields 3 (2 remain), $O_3$ yields 5 (1 remains, between \texttt{ideal\_semantic\_updater} and \texttt{mapping\_conditional\_shortcut}), and $O_4$ yields 7 singleton classes with \emph{no} estimand collisions ($\text{point\_identified}{=}1$). The trajectory is monotone: removing any axis from $O_4$ reintroduces at least one collision.

\subsection{Component essentiality via leave-one-out witnesses}
For each component $a$, we remove only $a$ from full support. Every leave-one-out support retains a cross-estimand collision (Table~\ref{tab:loo}), while full support has none. Target removal yields three collisions involving the ideal policy; each other removal yields one exclusive witness. Thus all four components are necessary for point identification of $\tau_{\mathrm{full}}$ on this policy class.

\begin{table}[t]
\centering
\footnotesize
\setlength{\tabcolsep}{3pt}
\begin{tabular}{lrrl}
\toprule
Removed & $|O|$ & Coll. & One witness against ideal \\
\midrule
Target & 24 & 3 & target inertia \\
Matched sham & 24 & 1 & any-edit reactor \\
Paired readout & 18 & 1 & generated-only failure \\
Full $S_3$ & 6 & 1 & mapping shortcut \\
\bottomrule
\end{tabular}
\caption{Leave-one-out audit from the canonical 16-subset lattice. Every reduced support has at least one cross-estimand collision; full support has projection size 36 and zero collisions.}
\label{tab:loo}
\end{table}

\subsection{Synthesis of minimum identifying support \texorpdfstring{$O^*$}{O*}}
The same collision structure that certifies under-identification also enables \emph{protocol synthesis}: instead of auditing a fixed support, find the smallest support that still separates every cross-estimand pair. For each pair $h,h'$ with $\tau(h)\neq\tau(h')$, let $D_{h,h'}=\{o\in O_{\mathrm{full}}: h(o)\neq h'(o)\}$ be its distinguishing cells. A support $S\subseteq O_{\mathrm{full}}$ point-identifies $\tau$ iff $S$ is a \emph{hitting set} of $\{D_{h,h'}\}$, i.e.\ $S\cap D_{h,h'}\neq\emptyset$ for every cross-estimand pair. We enumerate all minimum-size hitting sets exactly.

The synthesis procedure is the methodological contribution; its numerical outcome is case-specific. On the frozen seven-policy class it yields a striking result: $\lvert O^*\rvert=2$ cells---not 36. There are 26 distinct 2-cell supports (no cell is mandatory): every minimum pairs one \emph{target}-generated-choice cell with one \emph{sham}-candidate-scoring cell (target mapping in $\{1,3,5\}$ paired with any sham mapping, or $\{2,4\}$ paired with sham mapping $\geq2$; Appendix~\ref{app:ostar}). Removing either cell from $O^*$ reintroduces a collision. Under equal cell costs the minimal cost is 2; if candidate scoring costs three times a generated choice, every 2-cell support costs 4 (one scoring cell at 3 plus one generated cell at 1). Thus the audit does not demand the full tensor: identification has a precise price, and the framework both detects under-identification and synthesizes the cheapest support that removes it. The cost saving is not purely formal: scoring the same real responses under the 2-cell support $O^*$ still reproduces the accuracy--selective dissociation (gaps of $\approx$0.71 and $\approx$0.52 across the two models), while the strict 36-cell criterion is satisfied by 0/48 model--clusters, matching the full-support audit.

\subsection{Granularity bound on mixture recovery}
The stochastic mixture is recovered only at its declared granularity: the analyzer identifies the equivalence class containing the mixture (separating it from every deterministic alternative under $O_4$) but does not identify the exact mixture parameters from behavioral observation alone. This is a genuine identifiability bound, not an analyzer limitation: two mixtures with the same support-projected behavior are observationally equivalent by construction, so no estimator based on $O$ can separate them.

\section{Empirical Consequences on Frozen Model Responses}
\label{sec:real}

\subsection{Diagnostic case: original 24-cluster pilot}
The diagnostic case spans 24 clusters, six mappings, two readouts, three worlds, and two models: \texttt{Qwen/Qwen2.5-7B-Instruct} \citep{qwen2025technical} and \texttt{meta-llama/Llama-3.1-8B-Instruct} \citep{grattafiori2024llama3}. This gives 1,728 cells and 576 paired base/target/sham units. Target flips the oracle to UNKNOWN; matched sham preserves it. Generated choice parses free-form A/B/C answers, whereas candidate scoring uses conditional sequence log-probability. All intervals use 10,000 whole-cluster bootstrap replicates; contract cells are repeated measurements, not independent units.

\subsection{Two empirical estimands}
We distinguish semantic response fidelity from end-to-end contract fidelity. The former evaluates decoded semantic behavior conditional on a valid response; the latter additionally treats failures to produce a valid response under the required readout contract as failures. The paired-readout audit is therefore primarily an audit of end-to-end protocol dependence.
The frozen local estimand $\theta_{\mathrm{local}}$ is the paired selective-response rate over all base--target pairs: the proportion of 576 paired units where base and target are valid and oracle-correct and sham is valid and semantically equals its paired base response. Whole-cluster bootstrap resamples the 24 clusters and preserves all repeated contracts within each cluster. Separately, the empirical full-support criterion is evaluated on 48 model--clusters and requires the selective-response contract to hold across all 36 world--mapping--readout cells. This second criterion is the empirical counterpart of $\tau_{\mathrm{full}}$; it is not an ancillary sensitivity analysis of $\theta_{\mathrm{local}}$.

\begin{table}[t]
\centering
\scriptsize
\setlength{\tabcolsep}{2pt}
\begin{tabularx}{\columnwidth}{Ycc}
\toprule
Metric & Estimate [95\% CI] & $N$ \\
\midrule
Base accuracy & 0.403 --- & 576 \\
\textbf{Local response $\mid$ base-correct} & \textbf{0.138 [0.068, 0.218]} & 232 \\
Target-follow $\mid$ base-correct & 0.168 --- & 232 \\
Sham-stay $\mid$ base-correct & 0.918 --- & 232 \\
Local paired response & 0.056 [0.026, 0.090] & 576 \\
Full-support response & 0.000 --- & 48 \\
\bottomrule
\end{tabularx}
\caption{Estimand-separated results. The bold conditional headline is selective response among base-correct units (cluster bootstrap); target-follow and sham-stay decompose it. Local paired response and the full-support model--cluster criterion are distinct estimands.}
\label{tab:main}
\end{table}

\subsection{Conditional fidelity is the headline}
Because selective response requires base correctness, target following, and sham stability, the headline is conditional: only 0.138 of base-correct units satisfy it (32/232; cluster-bootstrap 95\% CI [0.068, 0.218]). Target following is 0.168, versus 0.918 sham stability, locating the shortfall mainly in target response. Separately, 0/48 model--clusters satisfy the full-support criterion. These rates do not become the binary synthetic property; they show why selective response must be measured rather than inferred from local correctness.

\subsection{Behavioral composition}
The five-class rule yields 317 inertia diagnoses and 183 invalid units, but does not require base correctness. A mutually exclusive audit instead finds 183 reportability failures, 179 base-knowledge failures, and 164 target-inertia cases; target inertia is therefore the largest \emph{base-correct semantic non-response}, not the dominant category overall. Sham stability holds for 349/576 units.

\subsection{Contract and transition robustness}
\label{sec:gate3}
We report three bounded robustness checks. Figure~\ref{fig:gate3} shows the balanced-transition analysis; Appendix~\ref{app:gate3} reports equal-validity contracts and bounded source transport. \textbf{Equal-validity contracts} give both constrained-generation variants pair-validity $1.0$; pooled candidate scoring has pair-validity $0.9826$. Every model--readout stratum nevertheless retains a positive accuracy--response gap. \textbf{Balanced transitions} cover six directions across 48 clusters: equal-weighted selective response is $0.324$ (95\% CI $[0.304,0.345]$), versus base accuracy $0.620$ ($[0.600,0.642]$), with $P(\Delta>0)=1.0$ over synchronized cluster bootstraps. Conditional response remains $0.244$--$0.259$ across the two mapped readouts. \textbf{Source transport} is limited to one deterministic second source (24 clusters), with rates $0.331$ ($[0.312,0.350]$) and $0.646$ ($[0.618,0.674]$).

A second solver-generated deterministic source uses an independent synthetic vocabulary and covers all six non-identity TRUE/FALSE/UNKNOWN transition directions, with four clusters per direction (24 total). We generated the full set, applied a deterministic shuffle with seed 20260806, and performed no result-based filtering. Intervals use 10,000 synchronized whole-cluster bootstrap replicates (seed 20260805); balanced pooling weights directions equally and resamples clusters within direction. Contract cells remain repeated measurements. These checks do not establish population transport or a pure target-edit effect.

\begin{figure*}[t]
\centering
\includegraphics[width=0.92\textwidth]{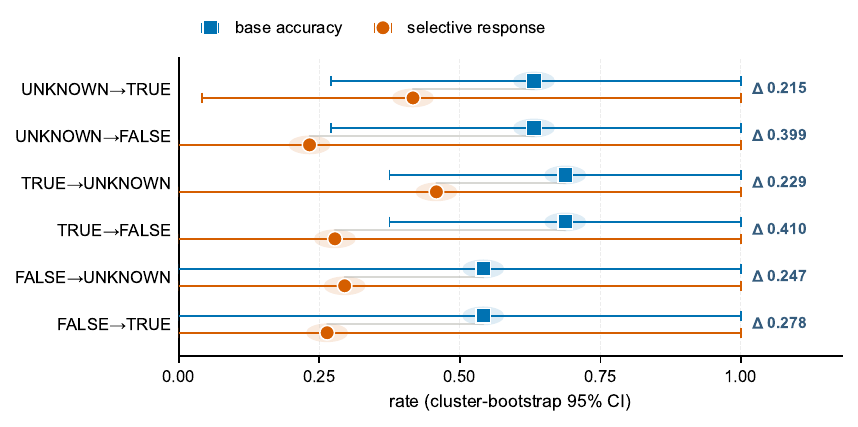}
\caption{Accuracy--response gaps across six balanced oracle-transition directions with synchronized cluster-bootstrap intervals. Squares show base accuracy; circles show selective response. These are measurement consequences, not pure edit effects or population transport.}
\label{fig:gate3}
\end{figure*}

\section{Behavioral Diagnoses Across Protocol Components}
\label{sec:taxonomy}

We localize observed contract failures without inferring internal causes; all decompositions are exploratory and post-result. Table~\ref{tab:composition} stratifies by model and readout, and Appendix Figures~\ref{fig:mapping}--\ref{fig:readout} audit mapping dependence and readout validity.

\noindent\textbf{Target and sham.} Target follow is $0.0677$; the exclusive audit finds 164 target-inertia and 179 base-knowledge failures. Sham stability is $0.6059$ (349/576); of the remaining 227 units, 183 are invalid and 44 valid but non-stable. Thus target and sham answer different diagnostic questions.

\noindent\textbf{Mapping and readout.} Semantic equivariance is $0.0625$ (18/288 groups), rising to $0.1023$ after validity filtering; fixed-label attachment is $0.0590$. Generated-choice validity is $0.4560$, versus $0.9942$ for scoring. When both readouts are valid, semantic answers always agree (base 132/132, target 125/125, sham 137/137). Thus the $115/288$ five-class agreement in Figure~\ref{fig:readout} reflects validity and diagnosis availability, not demonstrated semantic disagreement: paired readout matters mainly to end-to-end contract fidelity.

\noindent\textbf{Contract and transition.} Rates vary by model and readout, but base accuracy and local response induce the same four-contract ordering; we claim no ranking reversal. FALSE$\to$UNKNOWN and TRUE$\to$UNKNOWN rates are $0.0224$ and $0.0947$, an exploratory asymmetry.

\section{Related Work}
\label{sec:related}
We focus on one component of evaluation validity: identifiability relative to a declared behavioral class and support, which our framework makes mechanically auditable. Measurement work separates constructs from operationalizations and questions benchmark validity \citep{jacobs-wallach-2021-measurement,wallach-etal-2025-position,bowman-dahl-2021-will,bean2025measuring,siska-etal-2024-examining}. Behavioral tests and contrast sets show that accuracy can hide intervention failures \citep{ribeiro-etal-2020-beyond,gardner-etal-2020-evaluating}; counterfactual prompting motivates semantic-preserving baselines \citep{yang-etal-2026-compared}. Chain-of-thought work studies whether explanations reflect process \citep{turpin-etal-2023-language,lanham-etal-2023-measuring,yee-etal-2024-dissociation,paul-etal-2024-making,tutek-etal-2025-measuring,zaman-srivastava-2025-causal}. We instead use constructive collisions to audit point identification of black-box response fidelity, without inferring mechanisms.

\section{Discussion}
\label{sec:discussion}

Identification depends jointly on estimand, policy class, and support. On the frozen class, base-only observation collapses all policies; full support separates them; and each leave-one-out support retains a cross-estimand collision. The 5.56\% local rate and \textbf{0/48} full-support result answer different questions, preventing local success from implying contract stability.

For protocol design, declare the estimand, enumerate plausible collisions, and add separating measurements. For $\tau_{\mathrm{full}}$ on this class, all four support components have leave-one-out witnesses. Collision synthesis then converts measurement budget into a finite optimization problem: under the stated cell-cost model, a minimum separating support has two cells. This pre-inference audit complements broader construct-validity arguments.

\section{Conclusion}
For the frozen class, the audit exposes base-only collisions, verifies full-support identification, and synthesizes minimum separating support. Empirically, accuracy--response gaps persist under pair-valid generation, six balanced transitions, and one deterministic second source. Evaluation should state its behavioral estimand and test whether support identifies it before interpreting a score. The present study establishes this workflow in a controlled solver-grounded reasoning setting.

\section*{Limitations}
\label{sec:limitations}

The structural results are relative to a frozen seven-policy class, not an exhaustive model of natural-language behavior; additional policies may require more support. Solver-grounded signed-Horn data provide exact oracles and enable exhaustive audits; external validity beyond this controlled setting remains to be established. Evidence covers two instruction-tuned models, 24 diagnostic clusters, 48 balanced clusters, and one bounded synthetic source change, with cluster-level uncertainty; it does not establish transport to natural tasks or model populations. Target and sham differ in edit type and surface form, so their contrast supports protocol discrimination, not a pure edit effect. Local and full-support criteria use different units and cannot be pooled. Post-result behavioral diagnoses do not identify mechanisms.

\section*{Ethical Considerations}
\label{sec:ethics}

Explicit estimands and identification checks reduce over-interpretation of behavioral scores as evidence of ability or mechanism. Qwen2.5 and Llama 3.1 are used under their public licenses; inference is local and no weights are redistributed.

\section*{Data and Code Availability}
Code, data, and frozen evaluation artifacts will be released upon acceptance.

\bibliography{custom}

\begin{thebibliography}{21}
\providecommand{\natexlab}[1]{#1}

\bibitem[{Bean et~al.(2025)Bean, Kearns, Romanou, Hafner, Mayne, Batzner,
  Foroutan~Eghlidi, Schmitz, Korgul, Batra, Deb, Beharry, Emde, Foster, Gausen,
  Grandury, Han, Hofmann, Ibrahim, Kim, Kirk, Lin, Liu, Luettgau, Magomere,
  Rystr\o~m, Sotnikova, Yang, Zhao, Bibi, Bosselut, Clark, Cohan, Foerster,
  Gal, Hale, Raji, Summerfield, Torr, Ududec, Rocher, and
  Mahdi}]{bean2025measuring}
Andrew~M. Bean, Ryan~Othniel Kearns, Angelika Romanou, Franziska~Sofia Hafner,
  Harry Mayne, Jan Batzner, Negar Foroutan~Eghlidi, Chris Schmitz, Karolina
  Korgul, Hunar Batra, Oishi Deb, Emma Beharry, Cornelius Emde, Thomas Foster,
  Anna Gausen, Mar\'{\i}a Grandury, Sophia Han, Valentin Hofmann, Lujain
  Ibrahim, and 23 others. 2025.
\newblock \href {https://doi.org/10.52202/085713-0590} {Measuring what matters:
  Construct validity in large language model benchmarks}.
\newblock In \emph{Advances in Neural Information Processing Systems},
  volume~38. Curran Associates, Inc.

\bibitem[{Bowman and Dahl(2021)}]{bowman-dahl-2021-will}
Samuel~R. Bowman and George Dahl. 2021.
\newblock \href {https://doi.org/10.18653/v1/2021.naacl-main.385} {What will it
  take to fix benchmarking in natural language understanding?}
\newblock In \emph{Proceedings of the 2021 Conference of the North American
  Chapter of the Association for Computational Linguistics: Human Language
  Technologies}, pages 4843--4855, Online. Association for Computational
  Linguistics.

\bibitem[{Gardner et~al.(2020)Gardner, Artzi, Basmov, Berant, Bogin, Chen,
  Dasigi, Dua, Elazar, Gottumukkala, Gupta, Hajishirzi, Ilharco, Khashabi, Lin,
  Liu, Liu, Mulcaire, Ning, Singh, Smith, Subramanian, Tsarfaty, Wallace,
  Zhang, and Zhou}]{gardner-etal-2020-evaluating}
Matt Gardner, Yoav Artzi, Victoria Basmov, Jonathan Berant, Ben Bogin, Sihao
  Chen, Pradeep Dasigi, Dheeru Dua, Yanai Elazar, Ananth Gottumukkala, Nitish
  Gupta, Hannaneh Hajishirzi, Gabriel Ilharco, Daniel Khashabi, Kevin Lin,
  Jiangming Liu, Nelson~F. Liu, Phoebe Mulcaire, Qiang Ning, and 7 others.
  2020.
\newblock \href {https://doi.org/10.18653/v1/2020.findings-emnlp.117}
  {Evaluating models' local decision boundaries via contrast sets}.
\newblock In \emph{Findings of the Association for Computational Linguistics:
  EMNLP 2020}, pages 1307--1323, Online. Association for Computational
  Linguistics.

\bibitem[{Jacobs and Wallach(2021)}]{jacobs-wallach-2021-measurement}
Abigail~Z. Jacobs and Hanna Wallach. 2021.
\newblock \href {https://doi.org/10.1145/3442188.3445901} {Measurement and
  fairness}.
\newblock In \emph{Proceedings of the 2021 ACM Conference on Fairness,
  Accountability, and Transparency}, pages 375--385, New York, NY, USA.
  Association for Computing Machinery.

\bibitem[{Lanham et~al.(2023)Lanham, Chen, Radhakrishnan, Steiner, Denison,
  Hernandez, Li, Durmus, Hubinger, Kernion, Luko{\v{s}}i{\=u}t{\.e}, Nguyen,
  Cheng, Joseph, Schiefer, Rausch, Larson, McCandlish, Kundu, Kadavath, Yang,
  Henighan, Maxwell, Telleen-Lawton, Hume, Hatfield-Dodds, Kaplan, Brauner,
  Bowman, and Perez}]{lanham-etal-2023-measuring}
Tamera Lanham, Anna Chen, Ansh Radhakrishnan, Benoit Steiner, Carson Denison,
  Danny Hernandez, Dustin Li, Esin Durmus, Evan Hubinger, Jackson Kernion,
  Kamil{\.e} Luko{\v{s}}i{\=u}t{\.e}, Karina Nguyen, Newton Cheng, Nicholas
  Joseph, Nicholas Schiefer, Oliver Rausch, Robin Larson, Sam McCandlish,
  Sandipan Kundu, and 11 others. 2023.
\newblock \href {https://arxiv.org/abs/2307.13702} {Measuring faithfulness in
  chain-of-thought reasoning}.
\newblock \emph{arXiv preprint arXiv:2307.13702}.

\bibitem[{{Llama Team, AI @ Meta}(2024)}]{grattafiori2024llama3}
{Llama Team, AI @ Meta}. 2024.
\newblock \href {https://arxiv.org/abs/2407.21783} {The llama 3 herd of
  models}.
\newblock \emph{arXiv preprint arXiv:2407.21783}.

\bibitem[{Mirzadeh et~al.(2025)Mirzadeh, Alizadeh, Shahrokhi, Tuzel, Bengio,
  and Farajtabar}]{mirzadeh-etal-2024-gsm-symbolic}
Iman Mirzadeh, Keivan Alizadeh, Hooman Shahrokhi, Oncel Tuzel, Samy Bengio, and
  Mehrdad Farajtabar. 2025.
\newblock \href {https://arxiv.org/abs/2410.05229} {Gsm-symbolic: Understanding
  the limitations of mathematical reasoning in large language models}.
\newblock In \emph{International Conference on Learning Representations
  (ICLR)}.

\bibitem[{Paul et~al.(2024)Paul, West, Bosselut, and
  Faltings}]{paul-etal-2024-making}
Debjit Paul, Robert West, Antoine Bosselut, and Boi Faltings. 2024.
\newblock \href {https://doi.org/10.18653/v1/2024.findings-emnlp.882} {Making
  reasoning matter: Measuring and improving faithfulness of chain-of-thought
  reasoning}.
\newblock In \emph{Findings of the Association for Computational Linguistics:
  EMNLP 2024}, pages 15012--15032, Miami, Florida, USA. Association for
  Computational Linguistics.

\bibitem[{{Qwen Team}(2024)}]{qwen2025technical}
{Qwen Team}. 2024.
\newblock \href {https://arxiv.org/abs/2412.15115} {Qwen2.5 technical report}.
\newblock \emph{arXiv preprint arXiv:2412.15115}.

\bibitem[{Raji et~al.(2021)Raji, Denton, Bender, Hanna, and
  Paullada}]{raji-etal-2021-everything}
Inioluwa~Deborah Raji, Emily Denton, Emily~M. Bender, Alex Hanna, and
  Amandalynne Paullada. 2021.
\newblock \href {https://openreview.net/forum?id=zhahjq12BH} {{AI} and the
  everything in the whole wide world benchmark}.
\newblock In \emph{Advances in Neural Information Processing Systems: Datasets
  and Benchmarks Track}, volume~34. Curran Associates, Inc.

\bibitem[{Ribeiro et~al.(2020)Ribeiro, Wu, Guestrin, and
  Singh}]{ribeiro-etal-2020-beyond}
Marco~Tulio Ribeiro, Tongshuang Wu, Carlos Guestrin, and Sameer Singh. 2020.
\newblock \href {https://doi.org/10.18653/v1/2020.acl-main.442} {Beyond
  accuracy: Behavioral testing of {NLP} models with {C}heck{L}ist}.
\newblock In \emph{Proceedings of the 58th Annual Meeting of the Association
  for Computational Linguistics}, pages 4902--4912, Online. Association for
  Computational Linguistics.

\bibitem[{Salaudeen et~al.(2025)Salaudeen, Reuel, Ahmed, Bedi, Robertson,
  Sundar, Domingue, Wang, and Koyejo}]{salaudeen-etal-2025-measurement}
Olawale Salaudeen, Anka Reuel, Ahmed Ahmed, Suhana Bedi, Zachary Robertson,
  Sudharsan Sundar, Ben Domingue, Angelina Wang, and Sanmi Koyejo. 2025.
\newblock \href {https://arxiv.org/abs/2505.10573} {Measurement to meaning: A
  validity-centered framework for {AI} evaluation}.
\newblock \emph{arXiv preprint arXiv:2505.10573}.

\bibitem[{Siska et~al.(2024)Siska, Marazopoulou, Ailem, and
  Bono}]{siska-etal-2024-examining}
Charlotte Siska, Katerina Marazopoulou, Melissa Ailem, and James Bono. 2024.
\newblock \href {https://doi.org/10.18653/v1/2024.acl-long.560} {Examining the
  robustness of {LLM} evaluation to the distributional assumptions of
  benchmarks}.
\newblock In \emph{Proceedings of the 62nd Annual Meeting of the Association
  for Computational Linguistics (Volume 1: Long Papers)}, pages 10406--10421,
  Bangkok, Thailand. Association for Computational Linguistics.

\bibitem[{Turpin et~al.(2023)Turpin, Michael, Perez, and
  Bowman}]{turpin-etal-2023-language}
Miles Turpin, Julian Michael, Ethan Perez, and Samuel~R. Bowman. 2023.
\newblock \href {https://arxiv.org/abs/2305.04388} {Language models don't
  always say what they think: Unfaithful explanations in chain-of-thought
  prompting}.
\newblock In \emph{Advances in Neural Information Processing Systems 36
  (NeurIPS 2023)}.

\bibitem[{Tutek et~al.(2025)Tutek, Hashemi~Chaleshtori, Marasovic, and
  Belinkov}]{tutek-etal-2025-measuring}
Martin Tutek, Fateme Hashemi~Chaleshtori, Ana Marasovic, and Yonatan Belinkov.
  2025.
\newblock \href {https://doi.org/10.18653/v1/2025.emnlp-main.504} {Measuring
  chain of thought faithfulness by unlearning reasoning steps}.
\newblock In \emph{Proceedings of the 2025 Conference on Empirical Methods in
  Natural Language Processing}, pages 9935--9960, Suzhou, China. Association
  for Computational Linguistics.

\bibitem[{Wallach et~al.(2025)Wallach, Desai, Cooper, Wang, Atalla, Barocas,
  Blodgett, Chouldechova, Corvi, Dow, Garcia-Gathright, Olteanu, Pangakis,
  Reed, Sheng, Vann, Vaughan, Vogel, Washington, and
  Jacobs}]{wallach-etal-2025-position}
Hanna Wallach, Meera Desai, A.~Feder Cooper, Angelina Wang, Chad Atalla, Solon
  Barocas, Su~Lin Blodgett, Alexandra Chouldechova, Emily Corvi, P.~Alex Dow,
  Jean Garcia-Gathright, Alexandra Olteanu, Nicholas~J. Pangakis, Stefanie
  Reed, Emily Sheng, Dan Vann, Jennifer~Wortman Vaughan, Matthew Vogel, Hannah
  Washington, and Abigail~Z. Jacobs. 2025.
\newblock \href {https://proceedings.mlr.press/v267/wallach25a.html} {Position:
  Evaluating generative {AI} systems is a social science measurement
  challenge}.
\newblock In \emph{Proceedings of the 42nd International Conference on Machine
  Learning}, volume 267 of \emph{Proceedings of Machine Learning Research},
  pages 82232--82251. PMLR.

\bibitem[{Wei et~al.(2022)Wei, Wang, Schuurmans, Bosma, Ichter, Xia, Chi, Le,
  and Zhou}]{wei-etal-2022-cot}
Jason Wei, Xuezhi Wang, Dale Schuurmans, Maarten Bosma, Brian Ichter, Fei Xia,
  Ed~H. Chi, Quoc~V. Le, and Denny Zhou. 2022.
\newblock \href {https://arxiv.org/abs/2201.11903} {Chain-of-thought prompting
  elicits reasoning in large language models}.
\newblock In \emph{Advances in Neural Information Processing Systems 35
  (NeurIPS)}.

\bibitem[{Wu et~al.(2024)Wu, Qiu, Ross, Aky{\"u}rek, Chen, Wang, Kim, Andreas,
  and Kim}]{wu-etal-2024-reasoning}
Zhaofeng Wu, Linlu Qiu, Alexis Ross, Ekin Aky{\"u}rek, Boyuan Chen, Bailin
  Wang, Najoung Kim, Jacob Andreas, and Yoon Kim. 2024.
\newblock \href {https://doi.org/10.18653/v1/2024.naacl-long.102} {Reasoning or
  reciting? exploring the capabilities and limitations of language models
  through counterfactual tasks}.
\newblock In \emph{Proceedings of the 2024 Conference of the North American
  Chapter of the Association for Computational Linguistics: Human Language
  Technologies (Volume 1: Long Papers)}, pages 1819--1862, Mexico City, Mexico.
  Association for Computational Linguistics.

\bibitem[{Yang et~al.(2026)Yang, Levy, Goldberg, and
  Wallace}]{yang-etal-2026-compared}
Zihao Yang, Mosh Levy, Yoav Goldberg, and Byron~C. Wallace. 2026.
\newblock \href {https://arxiv.org/abs/2605.01048} {Compared to what? baselines
  and metrics for counterfactual prompting}.
\newblock \emph{arXiv preprint arXiv:2605.01048}.
\newblock Accepted at COLM 2026.

\bibitem[{Yee et~al.(2024)Yee, Li, Tang, Jung, Paturi, and
  Bergen}]{yee-etal-2024-dissociation}
Evelyn Yee, Alice Li, Chenyu Tang, Yeon~Ho Jung, Ramamohan Paturi, and Leon
  Bergen. 2024.
\newblock \href {https://arxiv.org/abs/2405.15092} {Dissociation of faithful
  and unfaithful reasoning in {LLM}s}.
\newblock \emph{arXiv preprint arXiv:2405.15092}.

\bibitem[{Zaman and Srivastava(2025)}]{zaman-srivastava-2025-causal}
Kerem Zaman and Shashank Srivastava. 2025.
\newblock \href {https://doi.org/10.18653/v1/2025.emnlp-main.1496} {A causal
  lens for evaluating faithfulness metrics}.
\newblock In \emph{Proceedings of the 2025 Conference on Empirical Methods in
  Natural Language Processing}, pages 29425--29449, Suzhou, China. Association
  for Computational Linguistics.

\end{thebibliography}

\appendix

\section{Model--Readout Behavioral Composition}
\label{app:composition}
\begin{table}[h]
\centering
\footnotesize
\setlength{\tabcolsep}{3pt}
\begin{tabular}{lrrrrr r}
\toprule
Contract & Sel. & Nonsel. & Perv. & Inert. & Inval. & Total \\
\midrule
Qwen $\cdot$ scoring & 17 & 2 & 18 & 107 & 0 & 144 \\
Qwen $\cdot$ choice & 9 & 2 & 12 & 83 & 38 & 144 \\
Llama $\cdot$ scoring & 6 & 3 & 7 & 123 & 5 & 144 \\
Llama $\cdot$ choice & 0 & 0 & 0 & 4 & 140 & 144 \\
\midrule
Total & 32 & 7 & 37 & 317 & 183 & 576 \\
\bottomrule
\end{tabular}
\caption{Five-class behavioral composition by model--readout contract. Labels are behavioral diagnoses; inertia does not imply base correctness.}
\label{tab:composition}
\end{table}

\section{Symbols and Policy Definitions}
\label{app:symbols}

Table~\ref{tab:symbols} summarizes the notation used in the main text. A \emph{policy} $h$ is a deterministic function from a world--mapping--readout triple $(w,\pi,r)$ to a raw response. The \emph{semantic response} is the raw response decoded under the contract defined by $\pi$; the \emph{oracle semantic response} is fixed by the independent solver before any model run.

\begin{figure}[t]
\centering
\includegraphics[width=\columnwidth]{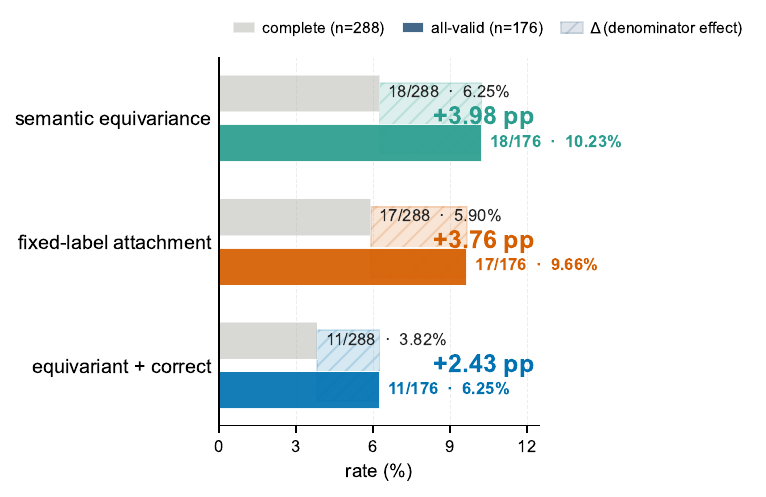}
\caption{Mapping audit: rates before ($n=288$) and after validity filtering ($n=176$). Numerators remain fixed for semantic equivariance, fixed-label attachment, and equivariance plus correctness, so each rate rises by 2.4--4.0 percentage points; denominator choice changes rates without changing counts.}
\label{fig:mapping}
\end{figure}

\begin{figure}[t]
\centering
\includegraphics[width=\columnwidth]{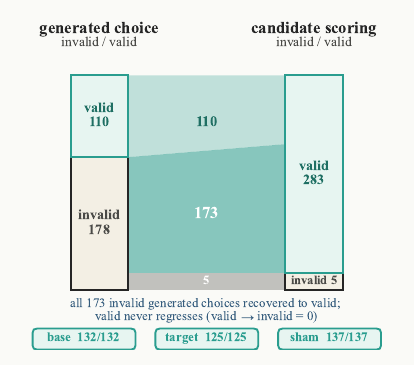}
\caption{Paired readout transition from generated choice to candidate scoring ($n=288$) as proportional flows: 173 invalid generated choices are recovered to valid by scoring and valid never regresses (0 valid$\to$invalid), alongside 110 valid--valid and 5 invalid--invalid pairs. Among both-valid responses, semantics agree completely in base, target, and sham worlds.}
\label{fig:readout}
\end{figure}

\begin{table}[h]
\centering
\footnotesize
\begin{tabularx}{\linewidth}{lY}
\toprule
Symbol & Meaning \\
\midrule
$\mathcal{H}$ & Finite behavioral policy class \\
$O$ & Observation support (set of projected responses) \\
$\tau_{\mathrm{full}}(h)\in\{0,1\}$ & Full-support contract-stable selective-response property \\
$\theta_{\mathrm{local}}$ & Local paired selective-response rate over model--cluster--mapping--readout units \\
$h\equiv_O h'$ & Observational equivalence on support $O$ \\
$w\in\{\text{base},\text{target},\text{sham}\}$ & Intervention world \\
$\pi\in S_3$ & Label permutation (bijection TRUE/FALSE/UNKNOWN $\to$ A/B/C) \\
$r\in\{\text{gc},\text{cs}\}$ & Readout: generated choice or candidate scoring \\
$R_{w,\pi,r}$ & Interventional response tensor cell \\
$O_k$ & Support after adding $k$ observation axes \\
\bottomrule
\end{tabularx}
\caption{Notation summary.}
\label{tab:symbols}
\end{table}

\section{Exclusive Witnesses for Axis Essentiality}
\label{app:witnesses}

For each support component $a$ we inspect policies with different $\tau_{\mathrm{full}}$ that separate under full support but collide under the corresponding leave-one-out support. The witnesses are realized by the seven deterministic policies of Section~\ref{sec:synthetic}:
\begin{itemize}
\item \textbf{Target update:} \texttt{ideal\_semantic\_updater} vs.\ \texttt{target\_inertia}. Both answer base correctly; only the target world separates them.
\item \textbf{Matched sham:} \texttt{ideal\_semantic\_updater} vs.\ \texttt{any\_edit\_reactor}. Both answer base and target in the same direction; only the sham world (where the reactor changes but the updater does not) separates them.
\item \textbf{Paired readout:} \texttt{ideal\_semantic\_updater} vs.\ \texttt{generated\_only\_failure}. The two readouts agree on base and target but disagree on whether the failure is readout-specific.
\item \textbf{Full $S_3$:} \texttt{ideal\_semantic\_updater} vs.\ \texttt{mapping\_conditional\_shortcut}. The shortcut updates only on a subset of label permutations; full $S_3$ exposes the dependence.
\end{itemize}
These witnesses make Proposition~\ref{prop:axis} constructive: omitting any axis reintroduces at least one collision.

\section{Complete Synthetic Recovery Decomposition}
\label{app:synthetic-full}

Table~\ref{tab:full-classes} reports all 16 subsets in the canonical lattice. Let T, S, R, and M denote target, sham, paired readout, and full $S_3$ mapping support. Only TSRM point-identifies $\tau_{\mathrm{full}}$.

\begin{table}[h]
\centering
\footnotesize
\setlength{\tabcolsep}{2.5pt}
\begin{tabular}{lrrr@{\quad}lrrr}
\toprule
Support & $|O|$ & Classes & Coll. & Support & $|O|$ & Classes & Coll. \\
\midrule
$\varnothing$ & 1 & 1 & 6 & TS & 3 & 3 & 2 \\
T & 2 & 2 & 3 & TR & 4 & 4 & 2 \\
S & 2 & 2 & 5 & TM & 12 & 4 & 2 \\
R & 2 & 2 & 5 & SR & 4 & 3 & 4 \\
M & 6 & 2 & 5 & SM & 12 & 3 & 4 \\
RM & 12 & 3 & 4 & TSR & 6 & 5 & 1 \\
TSM & 18 & 5 & 1 & TRM & 24 & 6 & 1 \\
SRM & 24 & 4 & 3 & TSRM & 36 & 7 & \textbf{0} \\
\bottomrule
\end{tabular}
\caption{Complete 16-subset support lattice. ``Coll.'' counts cross-estimand policy pairs. The four three-component rows are the leave-one-out supports.}
\label{tab:full-classes}
\end{table}

\section{Oracle-Transition Asymmetry}
\label{app:oracle-asymmetry}

Table~\ref{tab:oracle-asymmetry} reports the frozen primary estimand stratified by the base-to-target oracle transition. The data are exploratory and post-result. Selective response is more than four times more frequent after TRUE base oracles than after FALSE base oracles, suggesting that models update more readily when a previously entailed proposition becomes undetermined than when a previously refuted proposition does.

\begin{table}[h]
\centering
\footnotesize
\setlength{\tabcolsep}{2pt}
\begin{tabular}{lccc}
\toprule
Transition & Units & Sel.~rate & Sel.~if base-correct \\
\midrule
FALSE$\to$UNKNOWN & 312 & 0.022 & 0.065 \\
TRUE$\to$UNKNOWN & 264 & 0.095 & 0.202 \\
\bottomrule
\end{tabular}
\caption{Oracle-transition asymmetry. Intervals are omitted because this decomposition is exploratory and not part of the confirmatory estimand.}
\label{tab:oracle-asymmetry}
\end{table}

\section{Contract and Source Robustness}
\label{app:gate3}
\label{app:gate3-appendix}

\begin{figure*}[t]
\centering
\includegraphics[width=\textwidth]{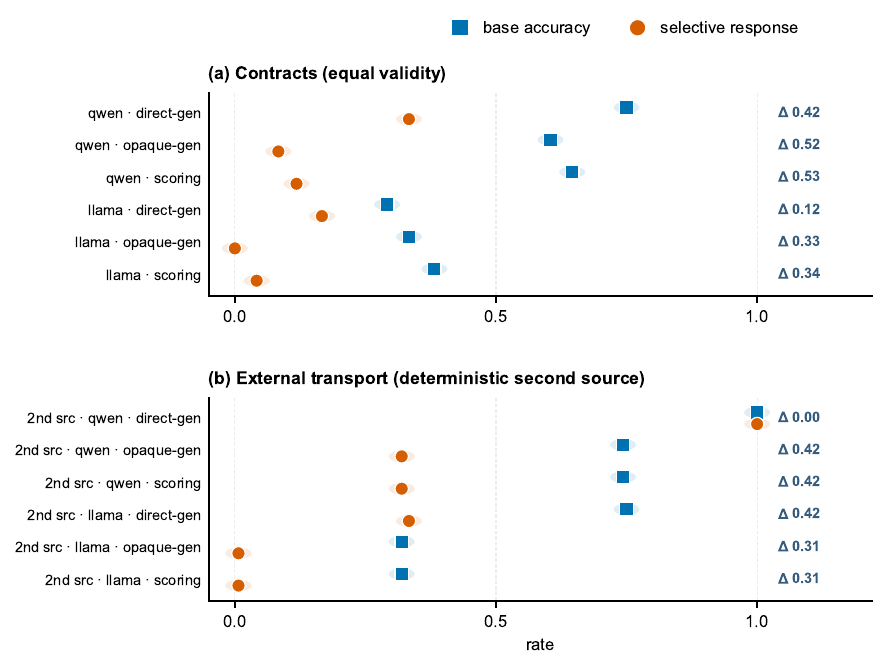}
\caption{Robustness readout contracts (equal validity) and bounded source transport: the accuracy--selective gap persists in every model--readout and source--model--readout stratum, with positive $\Delta$ annotations.}
\label{fig:gate3-appendix}
\end{figure*}

\section{Target--Sham Isolation Audit}
\label{app:isolation}
\begin{figure*}[t]
\centering
\includegraphics[width=\textwidth]{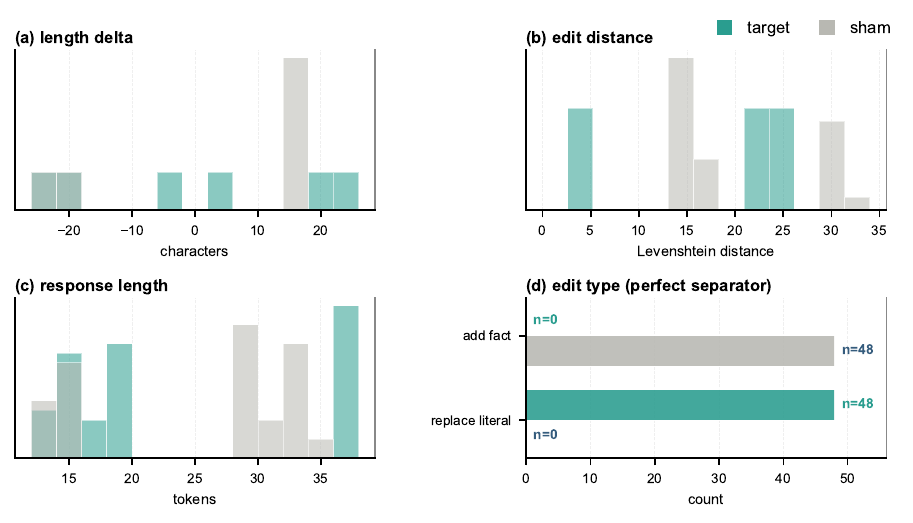}
\caption{Target--sham isolation audit: the formal oracle passes on 144 worlds; the surface audit flags material confounding, and the blinded human review passed with perfect inter-annotator agreement (\textbf{96/96}, Cohen's $\kappa=1.0$). The right panels show the target/sham surface imbalance that motivates the matched negative-control construction.}
\label{fig:isolation}
\end{figure*}

\section{Exclusive Behavioral Diagnoses}
\label{app:archetypes}

Figure~\ref{fig:archetypes} shows the mutually exclusive behavioral diagnoses over the 576 paired units. Reportability failure (183) is the largest category, followed by base-knowledge failure (179) and target inertia (164); target inertia is the largest base-correct semantic non-response category but not the dominant diagnosis overall.

\begin{figure}[t]
\centering
\includegraphics[width=0.82\columnwidth]{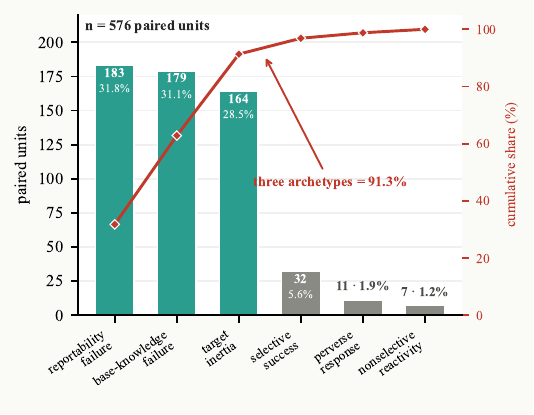}
\caption{Six mutually exclusive diagnoses over 576 paired units (counts in bars, cumulative share in red). Three archetypes---reportability failure, base-knowledge failure, and target inertia---account for 91.3\% of units. These post-result labels do not imply internal causes.}
\label{fig:archetypes}
\end{figure}

\section{Implementation Details}
\label{app:implementation}

\begingroup
\small
\paragraph{Solver grounding.} For each cluster, an independent solver computes the oracle label (TRUE, FALSE, or UNKNOWN) for the query with respect to the open-world signed-Horn theory. The target world is generated by adding a literal that flips an entailed or refuted query to UNKNOWN; the matched sham world applies a surface edit that preserves the oracle. All oracle labels are fixed before any model call.

\paragraph{Readouts.} The \emph{generated-choice} readout presents the query in free-form text and parses the model's output into one of A/B/C using a contract-specific prompt. The \emph{candidate-scoring} readout requests log-probabilities for the three candidates under the same contract and selects the highest. Both readouts are decoded into the semantic TRUE/FALSE/UNKNOWN space via the inverse of the active $S_3$ permutation.

\paragraph{Inference.} Both checkpoints are frozen at the versions reported in Section~\ref{sec:real}. Generated choice uses greedy decoding (\texttt{do\_sample=false}, \texttt{max\_new\_tokens=8}); candidate scoring selects the candidate with the largest conditional sequence log-probability sum. The real-model pilot required no hyperparameter search or fine-tuning. The synthetic policy-recovery study requires no model inference and enumerates the frozen policies exhaustively.

\paragraph{Attention-mask reproducibility check.} After correcting the Llama generated-choice attention-mask handling, a completed rerun produced 864 responses (432 regenerated choice responses plus 432 reused scoring responses) exactly equal to the prior recorded outputs, including all raw and parsed responses. This rules out the missing-mask warning as an implementation-level alternative explanation for the reported results.

\paragraph{Bootstrap.} Inference is performed at the cluster level via whole-cluster bootstrap (10,000 replicates). Cells within a cluster are repeated measurements under different contracts, not independent replicates, so the cluster is the resampling unit.
\endgroup

\section{Algorithms for Equivalence-Class and Collision Analysis}
\label{app:algorithms}
The synthetic recovery study is a pure enumeration over the frozen policy class. Algorithm~\ref{alg:classes} constructs observation profiles, partitions policies into observation-equivalence classes, and records every pair of policies that share a profile while differing on the target estimand. Algorithm~\ref{alg:axis} then checks each support component for essentiality by removing that component from the full support and testing whether any collision remains.

\begin{algorithm}[H]
\small
\caption{Equivalence classes and collisions}
\label{alg:classes}
\begin{algorithmic}[1]
\Require Policy class $\mathcal H$, target estimand $\tau:\mathcal H\to\{0,1\}$, observation support $O$
\Ensure Equivalence classes $\mathcal C$ and collision set $\mathsf{Coll}$
\For{$h \in \mathcal H$}
  \State $p[h] \gets \{(o, h(o)) : o \in O\}$ \Comment{observation profile}
\EndFor
\State $\mathcal C \gets$ partition of $\mathcal H$ by equality of $p[h]$
\State $\mathsf{Coll} \gets \{(h,h') \in \mathcal H^2 : \tau(h)\neq\tau(h') \land p[h]=p[h']\}$
\State \Return $(\mathcal C, \mathsf{Coll})$
\end{algorithmic}
\end{algorithm}

\begin{algorithm}[H]
\small
\caption{Check axis essentiality}
\label{alg:axis}
\begin{algorithmic}[1]
\Require $\mathcal H$, $\tau$, full support $O_4$, components $A=\{T,S,R,M\}$
\Require $T$: target; $S$: sham; $R$: paired readout; $M$: full-$S_3$ mapping
\Ensure Set $E$ of essential support components
\State $O_{-T}\gets O_{SRM}$; $O_{-S}\gets O_{TRM}$; $O_{-R}\gets O_{TSM}$; $O_{-M}\gets O_{TSR}$
\State $E \gets \emptyset$
\For{each support component $a \in A$}
  \State remove support component $a$ by selecting $O_{-a}$
  \State $(\mathcal C, \mathsf{Coll}) \gets \textsc{ComputeClasses}(\mathcal H,\tau,O_{-a})$
  \If{$\mathsf{Coll}\neq\emptyset$}
    \State mark $a$ essential; $E \gets E\cup\{a\}$
  \EndIf
\EndFor
\State \Return $E$
\end{algorithmic}
\end{algorithm}

Both algorithms are deterministic and run without any model call; their cost is $O(|\mathcal H|^2 \cdot |O_4|)$ for the full enumeration reported in Table~\ref{tab:synthetic}.

\section{Frozen Policy Definitions}
\label{app:policies}

Table~\ref{tab:policies} summarizes seven deterministic policies and the separately analyzed stochastic mixture. The deterministic policy class is used for all component-essentiality checks.

\begin{table}[H]
\centering
\footnotesize
\begin{tabularx}{\linewidth}{>{\raggedright\arraybackslash}p{2.6cm}cY}
\toprule
Policy & $\tau$ & Behavior on full $O_4$ \\
\midrule
\path{ideal_semantic_updater} & 1 & Correct on base; follows the target oracle; keeps sham identical to base; behaves consistently across both readouts and all six $S_3$ mappings. \\
\path{target_inertia} & 0 & Correct on base; repeats the base answer on target and sham regardless of oracle change. \\
\path{any_edit_reactor} & 0 & Changes answer whenever the input is edited, both on target and on matched sham. \\
\path{fixed_label_responder} & 0 & Always returns the same surface label; separable only once all six $S_3$ mappings are observed. \\
\path{mapping_conditional_shortcut} & 0 & Updates only on a subset of $S_3$ mappings; appears semantic on the observed mapping. \\
\path{generated_only_failure} & 0 & Fails semantically on generated choice while answering correctly under candidate scoring. \\
\path{scoring_only_success} & 0 & Answers correctly only under candidate scoring; the symmetric counterpart to the previous policy. \\
\path{stochastic_mixture} & 0--1 & Mixes ideal updater behavior with target inertia at a declared rate; recovered only at its mixture granularity. \\
\bottomrule
\end{tabularx}
\caption{Frozen behavioral policies in the synthetic recovery study. Policies are deterministic except for the stochastic mixture, which is a weighted combination of two deterministic policies.}
\label{tab:policies}
\end{table}

\newpage
\section{Interventional Response Tensor Schema}
\label{app:tensor}

The full observation support $O_4$ is the projection of the interventional response tensor $R_{w,\pi,r}$ onto all three indices. Table~\ref{tab:tensor-schema} lists the index ranges and their semantics; the product space contains $3\times6\times2=36$ cells per (model, cluster).

\section{Data Provenance}
\label{app:sources}

All reported numbers and figures are generated from frozen analysis outputs.

\begin{figure*}[t]
\section{Minimum-Support Synthesis Illustration}
\label{app:ostar}
\centering
\includegraphics[width=0.68\textwidth]{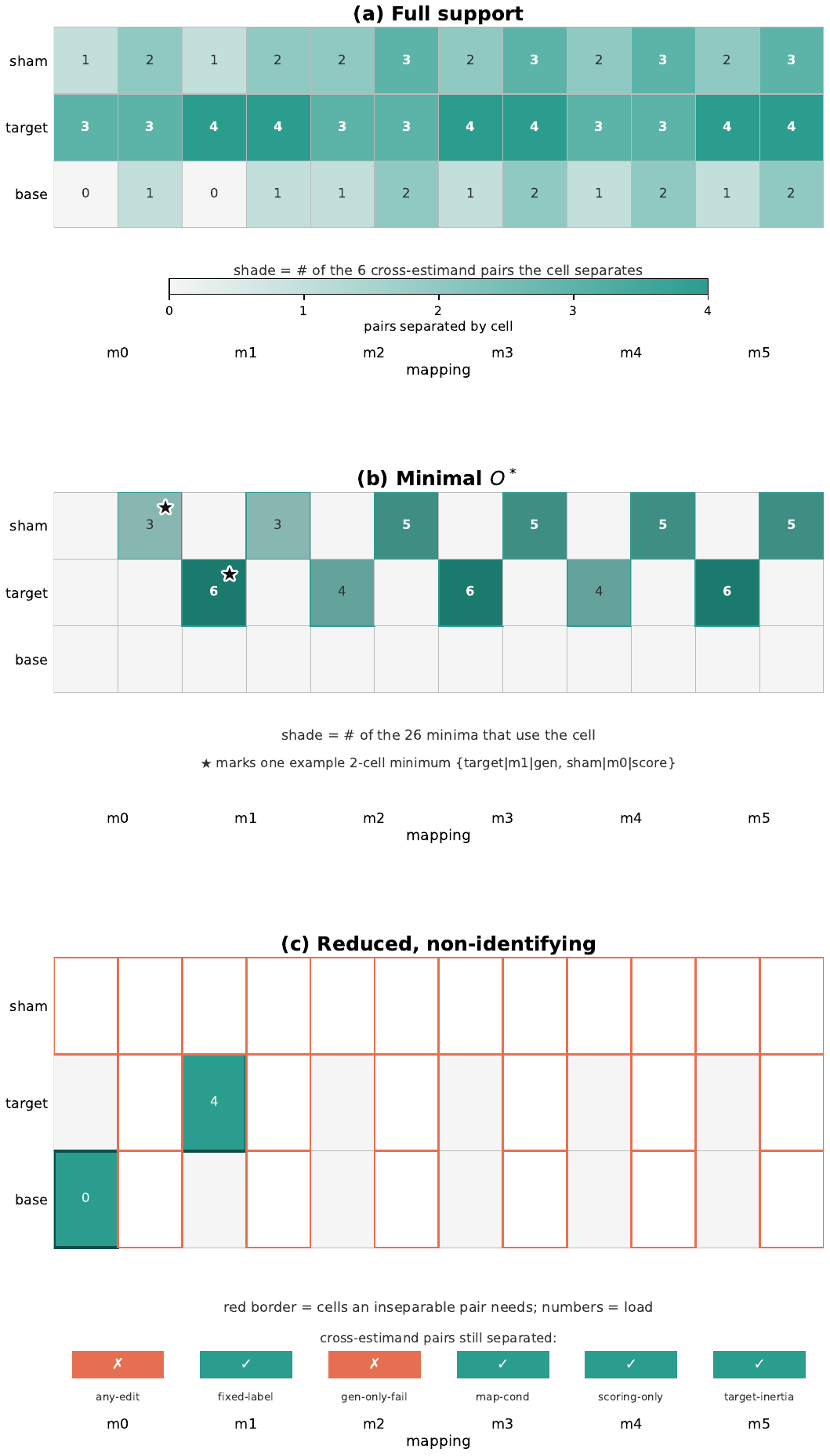}
\caption{Protocol synthesis on the 36-cell tensor. (a) Full support: each cell is shaded by its \emph{distinguishing load}---how many of the 6 cross-estimand pairs it separates (0--4). Target-world cells carry the most load; base-world generated-choice cells separate no pair. (b) Minimal identifying support $O^*$: all 26 two-cell minima are $\{$one \emph{target} generated-choice cell, one \emph{sham} candidate-scoring cell$\}$. Shade = how many of the 26 minima use the cell; no cell is mandatory and 25 of the 36 cells appear in no minimum. The starred pair is one example minimum. (c) A reduced 2-cell support $\{$target|m1|gen, base|m0|gen$\}$ is not identifying: the base|m0|gen cell separates no pair, so 2 of the 6 cross-estimand pairs (any-edit, gen-only-fail) remain inseparable (red); the bottom strip shows per-pair separation.}
\label{fig:ostar}
\end{figure*}

\begin{table}[H]
\centering
\footnotesize
\setlength{\tabcolsep}{2pt}
\begin{tabularx}{\linewidth}{llY}
\toprule
Index & Values & Meaning \\
\midrule
World $w$ & base, target, sham & Intervention applied to the query. Target flips oracle to UNKNOWN; sham preserves it. \\
Mapping $\pi$ & 6 permutations of $S_3$ & Bijection from semantic labels TRUE, FALSE, UNKNOWN to surface labels A, B, C. \\
Readout $r$ & gc, cs & \texttt{gc}: generated choice (free-form output parsed to A/B/C); \texttt{cs}: candidate scoring (log-probability over candidates). \\
\bottomrule
\end{tabularx}
\caption{Tensor indices for the interventional response tensor $R_{w,\pi,r}$.}
\label{tab:tensor-schema}
\end{table}

\end{document}